\documentclass{article}

\usepackage{microtype}
\usepackage{graphicx}
\usepackage{subfigure}
\usepackage{booktabs} % for professional tables
\usepackage{xcolor}
\usepackage{enumerate}

\usepackage{hyperref}

\usepackage[accepted]{icml2020}

\newcommand{\bb}{\mathbf{b}}
\newcommand{\be}{\mathbf{e}}
\newcommand{\bu}{\mathbf{u}}
\newcommand{\bv}{\mathbf{v}}
\newcommand{\bw}{\mathbf{w}}
\newcommand{\bx}{\mathbf{x}}

\newcommand{\bz}{\mathbf{z}}
\newcommand{\calD}{\mathcal{D}}

\usepackage{amsmath, amssymb, amsthm}

\newtheorem{theorem}{Theorem}
\newtheorem{corollary}{Corollary}

\icmltitlerunning{Certified Data Removal from Machine Learning Models}

\begin{document}

\twocolumn[
\icmltitle{Certified Data Removal from Machine Learning Models}

% It is OKAY to include author information, even for blind
% submissions: the style file will automatically remove it for you
% unless you've provided the [accepted] option to the icml2020
% package.

% List of affiliations: The first argument should be a (short)
% identifier you will use later to specify author affiliations
% Academic affiliations should list Department, University, City, Region, Country
% Industry affiliations should list Company, City, Region, Country

% You can specify symbols, otherwise they are numbered in order.
% Ideally, you should not use this facility. Affiliations will be numbered
% in order of appearance and this is the preferred way.
\icmlsetsymbol{equal}{*}

\begin{icmlauthorlist}
\icmlauthor{Chuan Guo}{cor}
\icmlauthor{Tom Goldstein}{face}
\icmlauthor{Awni Hannun}{face}
\icmlauthor{Laurens van der Maaten}{face}
\end{icmlauthorlist}

\icmlaffiliation{cor}{Department of Computer Science, Cornell University, New York, USA}
\icmlaffiliation{face}{Facebook AI Research, New York, USA}

\icmlcorrespondingauthor{Chuan Guo}{cg563@cornell.edu}
\icmlcorrespondingauthor{Laurens van der Maaten}{lvdmaaten@gmail.com}

% You may provide any keywords that you
% find helpful for describing your paper; these are used to populate
% the "keywords" metadata in the PDF but will not be shown in the document
\icmlkeywords{Machine Learning, ICML}

\vskip 0.3in
]

\printAffiliationsAndNotice{}

\begin{abstract}
Good data stewardship requires removal of data at the request of the data's owner. This raises the question if and how a trained machine-learning model, which implicitly stores information about its training data, should be affected by such a removal request. Is it possible to ``remove'' data from a machine-learning model? We study this problem by defining \emph{certified removal}: a very strong theoretical guarantee that a model from which data is removed cannot be distinguished from a model that never observed the data to begin with. We develop a certified-removal mechanism for linear classifiers and empirically study learning settings in which this mechanism is practical.
\end{abstract}

\section{Introduction}
Machine-learning models are often trained on third-party data, for example, many computer-vision models are trained on images provided by Flickr users \citep{thomee2016yfcc100m}. When a party requests that their data be removed from such online platforms, this raises the question how such a request should impact models that were trained prior to the removal. A similar question arises when a model is negatively impacted by a data-poisoning attack~\citep{biggio2012}. Is it possible to ``remove'' data from a model without re-training that model from scratch? 

%For example, given a trained classifier or prediction model, it may be possible to reconstruct (or obtain information about) data on which the model was trained~\citep{yeom2018privacy,carlini2019secret}. Is it possible to ``remove'' data from a model's ``memory''?

%\tom{Would be great to talk about the ``right to be forgotten'' and recent EU regs.  In fact, I like the concept of ``certified forgetting'' more than ``certified removal'' because of how it relates to recent internet and legal issues. }

We study this question in a framework we call \emph{certified removal}, which theoretically guarantees that an adversary cannot extract information about training data that was removed from a model. Inspired by differential privacy \citep{dwork2011differential}, certified removal bounds the max-divergence between the model trained on the dataset with some instances removed and the model trained on the dataset that never contained those instances. This guarantees that membership-inference attacks~\citep{yeom2018privacy,carlini2019secret} are unsuccessful on data that was removed from the model. We emphasize that certified removal is a very strong notion of removal; in practical applications, less constraining notions may equally fulfill the data owner's expectation of removal.

%Note that certified removal can be achieved without a differentially private model. For example, removal from a nearest-neighbor classifier is trivial even though such a classifier is clearly not differentially private. 

We develop a certified-removal mechanism for $L_2$-regularized linear models that are trained using a differentiable convex loss function, \emph{e.g.}, logistic regressors. Our removal mechanism applies a Newton step on the model parameters that largely removes the influence of the deleted data point; the residual error of this mechanism decreases quadratically with the size of the training set. To ensure that an adversary cannot extract information from the small residual (\emph{i.e.}, to certify removal), we mask the residual using an approach that randomly perturbs the training loss \citep{chaudhuri2011dperm}. We empirically study in which settings the removal mechanism is practical.

%We hope our definition of certified removal will inspire follow-up research and contribute to better data provenance in real-world applications of machine learning.

\section{Certified Removal}
\label{sec:def}
%!TEX root = paper.tex
Let $\calD$ be a fixed training dataset and let $A$ be a (randomized) learning algorithm that trains on $\calD$ and outputs a model $h \in \mathcal{H}$, that is, $A : \mathcal{D} \rightarrow \mathcal{H}$.
Randomness in $A$ induces a probability distribution over the models in the hypothesis set $\mathcal{H}$. We would like to remove a training sample, $\bx \in \calD$, from the output of $A$.

To this end, we define a data-removal mechanism $M$ that is applied to $A(\calD)$ and aims to remove the influence of $\bx$.
If removal is successful, the output of $M$ should be difficult to distinguish from the output of $A$ applied on $\calD \setminus \bx$. Given $\epsilon > 0$, we say that removal mechanism $M$ performs $\epsilon$-\emph{certified removal} ($\epsilon$-CR) for learning algorithm $A$ if $\forall \mathcal{T} \subseteq \mathcal{H}, \calD \subseteq \mathcal{X}, \bx \in \calD$:
\begin{equation}
e^{-\epsilon} \leq \frac{P(M(A(\calD), \calD, \bx) \in \mathcal{T})}{P(A(\calD \setminus \bx) \in \mathcal{T})} \leq e^\epsilon.
\label{eq:removal}
\end{equation}
This definition states that the ratio between the likelihood of (i) a model from which sample $\bx$ was removed and (ii) a model that was never trained on $\bx$ to begin with is close to one for all models in the hypothesis set, for all possible data sets, and for all removed samples. Note that although the definition requires that the mechanism $M$ is universally applicable to all training data points $\bx \in \calD$, it is also allowed to be data-dependent, \emph{i.e.}, both the training set $\calD$ and the data point to be removed $\bx$ are given as inputs to $M$.

We also define a more relaxed notion of $(\epsilon,\delta)$-\emph{certified removal} for $\delta > 0$ if $\forall \mathcal{T} \subseteq \mathcal{H}, \calD \subseteq \mathcal{X}, \bx \in \calD$:
\begin{align*}
P(M(A(\calD), \calD, \bx) \in \mathcal{T}) \leq e^\epsilon P(A(\calD \setminus \bx) \in \mathcal{T}) + \delta, \text{ and} \\
P(A(\calD \setminus \bx) \in \mathcal{T}) \leq e^\epsilon P(M(A(\calD), \calD, \bx) \in \mathcal{T}) + \delta.
\end{align*}
Conceptually, $\delta$ upper bounds the probability for the max-divergence bound in Equation \ref{eq:removal} to fail.

A trivial certified-removal mechanism $M$ with $\epsilon\!=\!0$ completely ignores $A(\calD)$ and evaluates $A(\calD \setminus \bx)$ directly, that is, it sets $M(A(\calD), \calD, \bx) \!=\! A(\calD \setminus \bx)$. Such a removal mechanism is impractical for many models, as it may require re-training the model from scratch every time a training sample is removed.
We seek mechanisms $M$ with removal cost $O(n)$ or less, with small constants, where $n = |\calD|$ is the training set size. % \chuan{We cannot achieve this at the moment}.

\paragraph{Insufficiency of parametric indistinguishability.}  One alternative relaxation of exact removal is by asserting that the output of $M(A(\calD), \calD, \bx)$ is sufficiently close to that of $A(\calD \setminus \bx)$.
%For learning algorithms that admit a global unique solution for every training set $\calD$ this is a natural consideration, and one may design suitable relaxations using statistical distances such as the Wasserstein distance to extend to cases where a global unique solution may not be achievable. However, we argue that even in natural scenarios, this definition does not provide sufficient utility.
It is easy to see that a model satsifying such a definition still retains information about $\bx$.  Consider a linear regressor trained on the dataset $\calD = \{(\be_1, 1), (\be_2, 2), \ldots, (\be_d, d)\}$ where $\be_i$'s are the standard basis vectors for $\mathbb{R}^d$.  A regressor that is initialized with zeros, or that has a weight decay penalty, will place a non-zero weight on $\bw_i$ if $(\be_i, i)$ is included in $\calD$, and a zero weight on $\bw_i$ if not.  In this case, an approximate removal algorithm that leaves $\bw_i$ small but non-zero still reveals that $(\be_i, i)$ appeared during training.
%Similar data leakage could occur, {\em e.g.}, in the case of a bag-of-words language model trained on a sentence containing rare words.
% However, if some data point $(\be_j, j)$ was never present in the training data then any reasonable learning algorithm trained on $\calD' = \calD \setminus (\be_j, j)$ must have a zero value at $\bw_j$\footnote{This is certainly true for any gradient-based algorithm with initialization to the all-zero vector, or if the loss function is suitably $L_1$- or $L_2$-regularized.} \tom{There's lots of models for which this isn't true.  An SGD with gaussian noise added, for example, won't satisfy this condition, even with L2 regularization and zero initialization.}. Then, if the result $\bw^-$ after removal does not have $\bw^-_j = 0$, we can conclude definitively that $(\be_j, j)$ was contained in the training data at some point, enabling a membership inference attack.

%We emphasize that this example is by no means artificially constructed. For example, bag-of-words models for text classification can encounter training data with certain words or word combinations that rarely occur but are very predictive of the output. Removal of such training points could be similar in nature to removal in the example above. To exacerbate the problem, certain rotational transformations may be applied to the training data to obfuscate the identification of these rarely occurring but indicative features. As a result, we opted for a definition of approximate removal that utilizes distributional divergences rather than measuring distance in parameter space.

\paragraph{Relationship to differential privacy.} Our formulation of certified removal is related to that of differential privacy \citep{dwork2011differential} but there are important differences. Differential privacy states that:
\begin{equation}
\forall \mathcal{T} \subseteq \mathcal{H}, \calD, \calD': e^{-\epsilon} \leq \frac{P(A(\calD) \in \mathcal{T})}{P(A(\calD') \in \mathcal{T})} \leq e^\epsilon,
\label{eq:diff_privacy}
\end{equation}
where $\calD$ and $\calD'$ differ in only one sample. Since $\calD$ and $\calD \setminus \bx$ only differ in one sample, it is straightforward to see that differential privacy of $A$ is a sufficient condition for certified removal, \emph{viz.}, by setting removal mechanism $M$ to the identity function. Indeed, if algorithm $A$ never memorizes the training data in the first place, we need not worry about removing that data.

Even though differential privacy is a sufficient condition, it is not a necessary condition for certified removal. For example, a nearest-neighbor classifier is not differentially private but it is trivial to certifiably remove a training sample in $O(1)$ time with $\epsilon = 0$. We note that differential privacy is a very strong condition, and most differentially private models suffer a significant loss in accuracy even for large $\epsilon$ \citep{chaudhuri2011dperm,abadi2016deep}. We therefore view the study of certified removal as analyzing the trade-off between utility and removal efficiency, with re-training from scratch and differential privacy at the two ends of the spectrum, and removal in the middle.

\section{Removal Mechanisms}
%!TEX root = paper.tex
We focus on certified removal from parametric models, as removal from non-parametric models ({\em e.g.,} nearest-neighbor classifiers) is trivial.
We first study linear models with strongly convex regularization before proceeding to removal from deep networks.

\subsection{Linear Classifiers}
Denote by $\calD = \{(\bx_1, y_1), \ldots, (\bx_n, y_n)\}$ the training set of $n$ samples, $\forall i: \bx_i \in \mathbb{R}^d$, with corresponding targets $y_i$. We assume learning algorithm $A$ tries to minimize the regularized empirical risk of a linear model:
\begin{equation*}
L(\bw; \calD) = \sum_{i=1}^n \ell(\bw^\top \bx_i, y_i) + \frac{\lambda n}{2} \|\bw\|_2^2,
\end{equation*}
where $\ell(\bw^\top \bx, y)$ is a convex loss that is differentiable everywhere. We denote $\bw^* \!=\! A(\calD) \!=\! \textrm{argmin}_\bw L(\bw; \calD)$ as it is uniquely determined.

Our approach to certified removal is as follows. We first define a removal mechanism that, given a training data point $(\bx,y)$ to remove, produces a model $\bw^-$ that is approximately equal to the unique minimizer of $L(\bw; \calD \setminus (\bx, y))$. The model produced by such a mechanism may still contain information about $(\bx,y)$. In particular, if the gradient $\nabla L(\bw^-; \calD \setminus (\bx, y))$ is non-zero, it reveals information about the removed data point. To hide this information, we apply a sufficiently large random perturbation to the model parameters at training time. This allows us to guarantee certified removal; the values of $\epsilon$ and $\delta$ depend on the approximation quality of the removal mechanism and on the distribution from which the model perturbation is sampled.

\paragraph{Removal mechanism.} We assume without loss of generality that we aim to remove the last training sample, $(\bx_n, y_n)$. Specifically, we define an efficient removal mechanism that approximately minimizes $L(\bw; \calD')$ with $\calD' = \calD \setminus (\bx_n,y_n)$.
First, denote the loss gradient at sample $(\bx_n, y_n)$ by $\Delta = \lambda \bw^* + \nabla \ell((\bw^*)^\top \bx_n, y_n)$ and the Hessian of $L(\cdot; \calD')$ at $\bw^*$ by $H_{\bw^*} = \nabla^2 L(\bw^*; \calD')$. We consider the \emph{Newton update removal mechanism} $M$:
\begin{equation}
\bw^- = M(\bw^*, \calD, (\bx_n, y_n)) := \bw^* + H_{\bw^*}^{-1} \Delta,
\label{eq:newton_update}
\end{equation}
which is a one-step Newton update applied to the gradient influence of the removed point $(\bx_n, y_n)$. The update $H_{\bw^*}^{-1} \Delta$ is also known as the influence function of the training point $(\bx_n,y_n)$ on the parameter vector $\bw^*$ \citep{cook1982residuals, koh2017understanding}.

The computational cost of this Newton update is dominated by the cost of forming and inverting the Hessian matrix. The Hessian matrix for $\calD$ can be formed offline with $O(d^2 n)$ cost. The subsequent Hessian inversion makes the removal mechanism $O(d^3)$ at removal time; the inversion can leverage efficient linear algebra libraries and GPUs.

To bound the approximation error of this removal mechanism, we observe that the quantity $\nabla L(\bw^-; \calD')$, which we refer to hereafter as the \emph{gradient residual}, is zero only when $\bw^-$ is the unique minimizer of $L(\cdot; \calD')$. We also observe that the gradient residual norm, $\|\nabla L(\bw^-; \calD')\|_2$, reflects the error in the approximation $\bw^-$ of the minimizer of $L(\cdot; \calD')$. We derive an upper bound on the gradient residual norm for the removal mechanism (\emph{cf.} Equation \ref{eq:newton_update}).\\

%\paragraph{Bounding the approximation error.}
%As argued in Section \ref{sec:def}, bounding the difference between $\bw^-$ and the minimizer for $L(\cdot; \calD')$ is not sufficient for guaranteeing that $\bw^-$ is indistinguishable from training from scratch on $\calD'$. However, we will show later on that by bounding the norm of the gradient $\nabla L(\bw^-; \calD')$ instead, we can design a suitable data removal mechanism that satisfies the $\epsilon$-certified removal definition. To this end, we first present an asymptotic result on the gradient norm $\| \nabla L(\bw^-; \calD') \|_2$.\\
%Since $L(\cdot, \calD')$ is strongly convex, there exists a unique minimizer for which the gradient $\nabla L(\cdot, \calD') \!=\! 0$. Hence, by bounding $\| \nabla L(\bw^-, \calD') \|_2$, we can analyze how close the approximate solution $\bw^-$ is to the global optimum.

\begin{theorem}
Suppose that $\forall (\bx_i, y_i) \in \calD, \bw \in \mathbb{R}^d: \|\nabla \ell(\bw^\top \bx_i, y_i)\|_2 \leq C$. Suppose also that $\ell''$ is $\gamma$-Lipschitz and $\| \bx_i \|_2 \leq 1$ for all $(\bx_i, y_i) \in \calD$. Then:
\begin{align}
\| \nabla L(\bw^-; \calD') \|_2 &= \|(H_{\bw_\eta} - H_{\bw^*}) H_{\bw^*}^{-1} \Delta\|_2 \label{eq:gradient_residual} \\
&\leq \gamma (n-1) \| H_{\bw^*}^{-1} \Delta \|_2^2 \leq \frac{4 \gamma C^2}{\lambda^2 (n-1)} \nonumber,
\end{align}
where $H_{\bw_\eta}$ denotes the Hessian of $L(\cdot; \calD')$ at the parameter vector $\bw_\eta = \bw^* + \eta H_{\bw^*}^{-1} \Delta$ for some $\eta \in [0,1]$.
\label{thm:asymptotic_bound}
\end{theorem}

\paragraph{Loss perturbation.}
Obtaining a small gradient norm $\| \nabla L(\bw^-; \calD') \|_2$ via Theorem \ref{thm:asymptotic_bound} does not guarantee certified removal. In particular, the direction of the gradient residual may leak information about the training sample that was ``removed.'' To hide this information, we use the loss perturbation technique of \citet{chaudhuri2011dperm} at training time. It perturbs the empirical risk by a random linear term:
\begin{equation*}
L_{\bb}(\bw; \calD) = \sum_{i=1}^n \ell\left(\bw^\top \bx_i, y_i\right) + \frac{\lambda n}{2} \|\bw\|_2^2 + \bb^\top \bw,
\end{equation*}
with $\bb \in \mathbb{R}^d$ drawn randomly from some distribution. We analyze how loss perturbation masks the information in the gradient residual $\nabla L_\bb(\bw^-; \calD')$ through randomness in $\bb$.

Let $A(\calD')$ be an exact minimizer\footnote{Our result can
be modified to work with approximate loss minimizers by incurring
a small additional error term.} for $L_{\bb}(\cdot; \calD')$ and let
$\tilde{A}(\calD')$ be an approximate minimizer of $L_{\bb}(\cdot; \calD')$, for example, our removal mechanism applied on the trained model. Specifically, let $\tilde{\bw}$ be
an approximate solution produced by $\tilde{A}$. This implies the gradient residual is:
\begin{equation}
\bu := \nabla L_{\bb}(\tilde{\bw}; \calD') = \sum_{i=1}^{n-1} \nabla \ell(\tilde{\bw}^\top \bx_i, y_i) + \lambda (n-1) \tilde{\bw} + \bb.
\label{eq:approximate_gradient}
\end{equation}
We assume that $\tilde{A}$ can produce a gradient residual $\bu$ with $\|\bu\|_2 \leq \epsilon'$
for some pre-specified bound $\epsilon'$ that is \emph{independent} of the perturbation vector $\bb$.

Let $f_A(\cdot)$ and $f_{\tilde{A}}(\cdot)$ be the density functions over the model
parameters produced by $A$ and $\tilde{A}$, respectively.
We bound the max-divergence between $f_A$ and $f_{\tilde{A}}$ for
any solution $\tilde{\bw}$ produced by approximate minimizer $\tilde{A}$.\\

\begin{theorem}
Suppose that $\bb$ is drawn from a distribution with density function $p(\cdot)$
such that for any $\bb_1, \bb_2 \in \mathbb{R}^d$ satisfying
$\|\bb_1 - \bb_2\|_2 \leq \epsilon'$, we have that:
$e^{-\epsilon} \leq \frac{p(\bb_1)}{p(\bb_2)} \leq e^{\epsilon}$.
Then $e^{-\epsilon} \leq \frac{f_{\tilde{A}}(\tilde{\bw})}{f_A(\tilde{\bw})} \leq e^\epsilon$
for any $\tilde{\bw}$ produced by $\tilde{A}$.
\label{thm:epsilon_cd}
\end{theorem}

\paragraph{Achieving certified removal.}
%Theorem \ref{thm:epsilon_cd} shows that if the removal mechanism has a gradient residual norm bound of $\epsilon'$,
%then by sampling $\bb$ appropriately, approximate optimization of the loss function
%$L_\bb$ results in an $\epsilon$-certified removal: the random perturbation of the empirical risk hides the information in the approximation performed by the removal mechanism.

%In particular, we specialize this result to the situation where $A$ samples a random vector $\bb$ and optimizes the loss $L_\bb$ exactly. The noise vector $\bb$ is discarded after optimization and is not observed by the adversary. The approximate optimizer $\tilde{A}$ is formed by the data-removal mechanism in Equation~\ref{eq:newton_update}, \emph{i.e.}, $\tilde{A} = M \circ A$. Note that the Newton update remains the same for loss $L_{\bb}$ because the Hessian matrix $H_{\bw^*}$ is unaffected by the linear term $\bb^\top \bw$. The gradient residual bound is given by Theorem \ref{thm:asymptotic_bound}.

We can use Theorem \ref{thm:epsilon_cd} to prove certified removal by combining it with the gradient residual norm bound $\epsilon'$ from Theorem \ref{thm:asymptotic_bound}. The security parameters $\epsilon$ and $\delta$ depend on the distribution from which $\bb$ is sampled. We state the guarantee of $(\epsilon,\delta)$-certified removal below for two suitable distributions $p(\bb)$.\\

\begin{theorem}
Let $A$ be the learning algorithm that returns the unique optimum of the loss $L_{\bb}(\bw; \calD)$ and let $M$ be the Newton update removal mechanism (cf., Equation~\ref{eq:newton_update}). Suppose that $\| \nabla L(\bw^-; \calD') \|_2 \leq \epsilon'$ for some computable bound $\epsilon' > 0$. We have the following guarantees for $M$:
\begin{enumerate}[(i)]
\setlength\itemsep{-1ex}
\vspace{-2ex}
\item If $\bb$ is drawn from a distribution with density $p(\bb) \propto e^{-\frac{\epsilon}{\epsilon'} \|\bb\|_2}$, then $M$ is $\epsilon$-CR for $A$;
\item If $\bb \sim \mathcal{N}(0, c \epsilon' / \epsilon)^d$ with $c > 0$, then $M$ is $(\epsilon,\delta)$-CR for $A$ with $\delta = 1.5 \cdot e^{-c^2/2}$.
\end{enumerate}
\label{thm:removal_main}
\end{theorem}
The distribution in (i) is equivalent to sampling a direction uniformly over the unit sphere and sampling a norm from the $\Gamma(d, \frac{\epsilon'}{\epsilon})$ distribution \cite{chaudhuri2011dperm}.

\subsection{Practical Considerations}

\paragraph{Least-squares and logistic regression.}
The certified removal mechanism described above can be used with least-squares and logistic regression, which are ubiquitous in real-world applications of machine learning.

Least-squares regression assumes $\forall i: y_i \in \mathbb{R}$ and uses the loss function $\ell(\bw^\top \bx_i, y_i) = \left(\bw^\top \bx_i - y_i\right)^2$. The Hessian of this loss function is $\nabla^2 \ell(\bw^\top \bx_i, y_i) = \bx_i \bx_i^\top$, which is independent of $\bw$. In particular, the gradient residual from Equation \ref{eq:gradient_residual} is exactly zero, which makes the Newton update in Equation~\ref{eq:newton_update} an $\epsilon$-certified removal mechanism with $\epsilon=0$ (loss perturbation is not needed). This is not surprising since the Newton update assumes a local quadratic approximation of the loss, which is exact for least-squares regression.

Logistic regression assumes $\forall i: y_i \in \{-1, +1\}$ and uses the loss function
$\ell(\bw^\top \bx_i, y_i) = -\log \sigma\left(y_i \bw^\top \bx_i\right)$,
where $\sigma(\cdot)$ denotes the sigmoid function, $\sigma(x) = \frac{1}{1 + \exp(-x)}$. The loss gradient and Hessian are given by:
\begin{align*}
\nabla \ell(\bw^\top \bx_i, y_i) &= \left(\sigma(y_i \bw^\top \bx_i) - 1\right) y_i \bx_i\\
\nabla^2 \ell(\bw^\top \bx_i, y_i) &= \sigma(y_i \bw^\top \bx_i) \left(1 - \sigma(y_i \bw^\top \bx_i)\right) \bx_i \bx_i^\top.
\end{align*}
Under the assumption that $\| \bx_i \|_2 \leq 1$ for all $i$, it is straightforward to show that $\| \nabla \ell(\bw^\top \bx_i, y_i) \|_2 \leq 1$ and that $\ell''(\bw^\top \bx_i, y_i)$ is $\gamma$-Lipschitz with $\gamma = 1/4$. This allows us to apply Theorem \ref{thm:asymptotic_bound} to logistic regression.

\paragraph{Data-dependent bound on gradient norm.}
The bound in Theorem \ref{thm:asymptotic_bound} contains a constant factor of $1/\lambda^2$ and may be too loose for practical applications on smaller datasets. Fortunately, we can derive a \emph{data-dependent} bound on the gradient residual that can be efficiently computed and that is much tighter in practice. Recall that the Hessian of $L(\cdot; \calD')$ has the form:
$$H_{\bw} = (X^-)^\top D_{\bw} X^- + \lambda (n-1) I_d,$$
where $X^- \in \mathbb{R}^{(n-1) \times d}$ is the data matrix corresponding to $\calD'$, $I_d$ is the identity matrix of size $d\!\times\!d$, and $D_{\bw}$ is a diagonal matrix containing values:
$$(D_{\bw})_{ii} = \ell''\left(\bw^\top \bx_i, y_i\right).$$
By Equation \ref{eq:gradient_residual} we have that:
\begin{align*}
\| \nabla L(\bw^-; \calD') \|_2 &= \| (H_{\bw_\eta} - H_{\bw^*}) H_{\bw^*}^{-1} \Delta \|_2 \\
&= \| (X^-)^\top (D_{\bw_\eta} - D_{\bw^*}) X^- H_{\bw^*}^{-1} \Delta \|_2 \\
&\leq \| X^- \|_2 \| D_{\bw_\eta} - D_{\bw^*} \|_2 \| X^- H_{\bw^*}^{-1} \Delta \|_2.
\end{align*}
The term $\| D_{\bw_\eta} - D_{\bw^*} \|_2$ corresponds to the maximum singular value of a diagonal matrix, which in turn is the maximum value among its diagonal entries. Given the Lipschitz constant $\gamma$ of the second derivative $\ell''$, we can thus bound it as:
$$\| D_{\bw_\eta} - D_{\bw^*} \|_2 \leq \gamma \| \bw_\eta - \bw^* \|_2 \leq \gamma \| H_{\bw^*}^{-1} \Delta \|_2.$$\
The following corollary summarizes this derivation.\\

\begin{corollary}
The Newton update $\bw^- = \bw^* + H_{\bw^*}^{-1} \Delta$ satisfies:
$$\| \nabla L(\bw^-; \calD') \|_2 \leq \gamma \| X^- \|_2 \| H_{\bw^*}^{-1} \Delta \|_2 \| X^- H_{\bw^*}^{-1} \Delta \|_2,$$
where $\gamma$ is the Lipschitz constant of $\ell''$.
\label{cor:data_dependent}
\end{corollary}

The bound in Corollary \ref{cor:data_dependent} can be easily computed from the Newton update $H_{\bw^*}^{-1} \Delta$ and the spectral norm of $X^-$, the latter admitting efficient algorithms such as power iterations.
%\chuan{We cannot pre-compute $X^- H_{\bw^*}^{-1} \Delta$, so the removal is still $O(n)$. One possible fix is that we compute the SVD $X^- = U \Sigma V$ so that $\| X^- H_{\bw^*}^{-1} \Delta \|_2 = \| \Sigma V H_{\bw^*}^{-1} \Delta \|_2$, which only involves $d \times d$ matrices. The SVD can be updated efficiently by numerically eigendecomposing a $d \times d$ matrix. If we want to propose this idea then we should add another section on off-line computation.}

\paragraph{Multiple removals.} The worst-case gradient residual norm after $T$
removals can be shown to scale linearly in $T$. We can prove this using induction
on $T$. The base case, $T\!=\!1$, is proven above. Suppose that the gradient residual
after $T \!\geq\! 1$ removals is $\bu_T$ with $\| \bu_T \|_2 \leq T \epsilon'$,
where $\epsilon'$ is the gradient residual norm bound for a single removal.
Let $\calD^{(T)}$ be the training dataset with $T$ data points removed.
Consider the modified loss function $L_\bb^{(T)}(\bw; \calD^{(T)}) = L_\bb(\bw; \calD^{(T)}) - \bu_T ^\top \bw$
and let $\bw_T$ be the approximate solution after $T$ removals. Then $\bw_T$
is an exact solution of $L_\bb^{(T)}(\bw; \calD^{(T)})$, hence, the argument above can be applied
to $L_\bb^{(T)}(\bw; \calD^{(T)})$ to show that the Newton update approximation $\bw_{T+1}$ has
gradient residual $\bu'$ with norm at most $\epsilon'$. Then:
\begin{align*}
\bu' = \nabla_\bw L_\bb^{(T)}(\bw; \calD^{(T)}) = \nabla_\bw L_\bb(\bw; \calD^{(T)}) - \bu_T \\
\Rightarrow 0 = \nabla_\bw L_\bb(\bw) - \bu_T - \bu'.
\end{align*}
Thus the gradient residual for $L_\bb(\bw; \calD^{(T)})$ after $T+1$ removals is
$\bu_{T+1} := \bu_T + \bu'$ and its norm is at most $(T+1) \epsilon'$ by the
triangle inequality.

%\laurens{We need to point out here that the $O(d^3)$ compute at removal time can only be achieved once, unless we can somehow approximate the new Hessian from the old one and still maintain the $\epsilon$-CR property. This drawback may also motivate batch removal. Batch removal is actually not that unreasonable given that removal guarantees are generally phrased as ``we will remove your data within $k$ days''.}

\paragraph{Batch removal.} In certain scenarios, data removal may not need to occur immediately after the data's owner requests removal. This potentially allows for batch removals in which multiple training samples are removed at once for improved efficiency. The Newton update removal mechanism naturally supports this extension. Assume without loss of generality that the batch of training data to be removed is $\mathcal{D}_m = \{(\bx_{n-m+1}, y_{n-m+1}), \ldots, (\bx_n, y_n)\}$. Define:
\begin{align*}
\Delta^{(m)} &= m\lambda \bw^* + \sum_{j=n-m+1}^n \nabla \ell((\bw^*)^\top \bx_j, y_j)\\
H_{\bw^*}^{(m)} &= \nabla^2 L(\bw^*; \mathcal{D} \setminus \mathcal{D}_m).
\end{align*}
The batch removal update is:
\begin{equation}
\bw^{(-m)} = \bw^* + \left[ H_{\bw^*}^{(m)} \right]^{-1} \Delta^{(m)}.
\label{eq:newton_update_batch}
\end{equation}

We derive bounds on the gradient residual norm for batch removal that are similar Theorem \ref{thm:asymptotic_bound} and Corollary \ref{cor:data_dependent}.\\

\begin{theorem}
Under the same regularity conditions of Theorem \ref{thm:asymptotic_bound}, we have that:
\begin{align*}
\| \nabla L(\bw^{(-m)}; \calD \setminus \calD_m) \|_2 &\leq \gamma (n-m) \left\| \left[ H_{\bw^*}^{(m)} \right]^{-1} \Delta^{(m)} \right\|_2^2 \\
&\leq \frac{4 \gamma m^2 C^2}{\lambda^2 (n-m)}.
\end{align*}
\label{thm:asymptotic_bound_batch}
\end{theorem}

\begin{corollary}
The Newton update $\bw^{(-m)} = \bw^* + \left[ H_{\bw^*}^{(m)} \right]^{-1} \Delta^{(m)}$ satisfies:
\begin{align*}
&\| L(\bw^{(-m)}; \calD \setminus \calD_m) \|_2 \leq \\
& \gamma \| X^{(-m)} \|_2\! \left\| \left[ H_{\bw^*}^{(m)} \right]^{-1} \!\!\Delta^{(m)} \right\|_2 \! \left\| X^{(-m)} \! \left[ H_{\bw^*}^{(m)} \right]^{-1} \!\!\Delta^{(m)} \right\|_2\!,
\end{align*}
where $X^{(-m)}$ is the data matrix for $\calD \setminus \calD_m$ and $\gamma$ is the Lipschitz constant of $\ell''$.
\label{cor:data_dependent_batch}
\end{corollary}

\begin{algorithm}[t]
\caption{Training of a certified removal-enabled model.}
\label{alg:training}
\begin{algorithmic}[1]
\STATE \textbf{Input}: Dataset $\calD$, loss $\ell$, parameters $\sigma,\lambda>0$.
\STATE Sample $\bb \sim \mathcal{N}(0, \sigma)^d$.
\STATE \textbf{Return}: $\text{argmin}_{\bw \in \mathbb{R}^d} \scriptstyle{\sum_{i=1}^n \ell(\bw^\top \bx_i, y_i) + \lambda n \| \bw \|_2^2 + \bb^\top \bw}$.
\end{algorithmic}
\end{algorithm}

\begin{algorithm}[t]
\caption{Repeated certified removal of data batches.}
\label{alg:removal}
\begin{algorithmic}[1]
\STATE \textbf{Input}: Dataset $\calD$, loss $\ell$, parameters $\epsilon,\delta,\sigma,\lambda>0$.
\STATE \hspace{6ex} Lipschitz constant $\gamma$ of $\ell''$.
\STATE \hspace{6ex} Solution $\bw$ computed by Algorithm \ref{alg:training}.
\STATE \hspace{6ex} Sequence of batches of training sample
\STATE \hspace{6ex} indices to be removed: $B_1, B_2, \ldots$
\STATE Gradient residual bound $\beta \gets 0$.
\STATE $c \gets \sqrt{2 \log(1.5/\delta)}$.
\STATE $K \gets \sum_{i=1}^n \bx_i \bx_i^\top$.
\STATE $X \gets [\bx_1 | \bx_2 | \cdots | \bx_n]^\top$.
\FOR{$j=1,2,\ldots$}
\STATE $\Delta \gets |B_j| \lambda \bw + \sum_{i \in B_j} \nabla \ell(\bw^\top \bx_i, y_i)$.
\STATE $H \gets \sum_{i : i \notin B_1,B_2,\ldots,B_j} \nabla^2 \ell(\bw^\top \bx_i, y_i)$.
\STATE $X \gets \texttt{remove\_rows}(X, B_j)$.
\STATE $K \gets K - \sum_{i \in B_j} \bx_i \bx_i^\top$.
\STATE $\beta \gets \beta + \gamma \sqrt{\| K \|_2} \cdot \|H^{-1} \Delta\|_2 \cdot \|XH^{-1} \Delta\|_2$.
\IF{$\beta > \sigma \epsilon / c$}
\STATE Re-train from scratch using Algorithm \ref{alg:training}.
\ELSE
\STATE $\bw \gets \bw + H^{-1} \Delta$.
\ENDIF
\ENDFOR
\end{algorithmic}
\end{algorithm}

Interestingly, the gradient residual bound in Theorem \ref{thm:asymptotic_bound_batch} scales quadratically with the number of removals, as opposed to linearly when removing examples one-by-one. This increase in error is due to a more crude approximation of the Hessian, that is, we compute the Hessian only once at the current solution $\bw^*$ rather than once per removed data. %This demonstrates an inherent trade-off between accuracy and efficiency when removing multiple data points.

\paragraph{Reducing online computation.} The Newton update requires forming and inverting the Hessian. Although the $O(d^3)$ cost of inversion is relatively limited for small $d$ and inversion can be done efficiently on GPUs, the cost of forming the Hessian is $O(d^2 n)$, which may be problematic for large datasets. However, the Hessian can be formed at training time, \emph{i.e.}, before the data to be removed is presented, and only the inverse needs to be computed at removal time.

When computing the data-dependent bound, a similar technique can be used for calculating the term $\| X^- H_{\bw^*}^{-1} \Delta \|_2$ -- which involves the product of the $(n-1) \times d$ data matrix $X^-$ with a $d$-dimensional vector. We can reduce the online component of this computation to $O(d^3)$ by forming the SVD of $X$ offline and applying online down-dates \citep{gu1995downdating} to form the SVD of $X^-$ by solving an eigen-decomposition problem on a $d \times d$ matrix. It can be shown that this technique reduces the computation of $\| X^- H_{\bw^*}^{-1} \Delta \|_2$ to involve only $d \times d$ matrices and $d$-dimensional vectors, which enables the online computation cost to be independent of $n$.

\paragraph{Pseudo-code.} We present pseudo-code for training removal-enabled models and for the $(\epsilon,\delta)$-CR Newton update mechanism. During training (Algorithm \ref{alg:training}), we add a random linear term to the training loss by sampling a Gaussian noise vector $\bb$. The choice of $\sigma$ determines a ``removal budget'' according to Theorem \ref{thm:removal_main}:  the maximum gradient residual norm that can be incurred is $\sigma \epsilon / c$. When optimizing the training loss, any optimizer with convergence guarantee for strongly convex loss functions can be used to find the minimizer in Algorithm \ref{alg:training}. We use L-BFGS \citep{liu1989lbfgs} in our experiments as it was the most efficient of the optimizers we tried.

During removal (line 19 in Algorithm \ref{alg:removal}), we apply the batch Newton update (Equation \ref{eq:newton_update_batch}) and compute the gradient residual norm bound using Corollary \ref{cor:data_dependent_batch} (line 15 in Algorithm \ref{alg:removal}). The variable $\beta$ accumulates the gradient residual norm over all removals. If the pre-determined budget of $\sigma \epsilon / c$ is exceeded, we train a new removal-enabled model from scratch using Algorithm \ref{alg:training} on the remaining data points.

\begin{table*}[t]
\centering
\resizebox{\textwidth}{!}{
\begin{tabular}{l|cccc}
\toprule
\textbf{Dataset} & \bf MNIST (\S\ref{sec:linear}) & \bf LSUN (\S\ref{sec:text_and_image}) & \bf SST (\S\ref{sec:text_and_image}) & \bf SVHN (\S\ref{sec:non-linear}) \\
\midrule
\textbf{Removal setting} & CR Linear & Public Extractor + CR Linear & Public Extractor + CR Linear & DP Extractor + CR Linear \\
%\textbf{Re-training time} & 4.7s & 58.1s & 8.0s & 1.5h+28.8s \\
\textbf{Removal time} & 0.04s & 0.48s & 0.07s & 0.27s \\
\midrule
\textbf{Training time} & 15.6s & 124s & 61.5s & 1.5h \\
\bottomrule
\end{tabular}
}
\vspace{-2ex}
\caption{\textbf{Summary of removal and training times observed in our experiments.} For LSUN and SST, the public extractor is trained on a public dataset and hence removal is only applied to the linear model. For SVHN, removal is applied to a linear model that operates on top of a differentially private feature extractor. In all cases, using the Newton update to (certifiably) remove data is several orders of magnitude faster than re-training the model from scratch.}
\label{tab:timings}
\end{table*}

\begin{figure*}[t!]
\centering
\includegraphics[width=\textwidth]{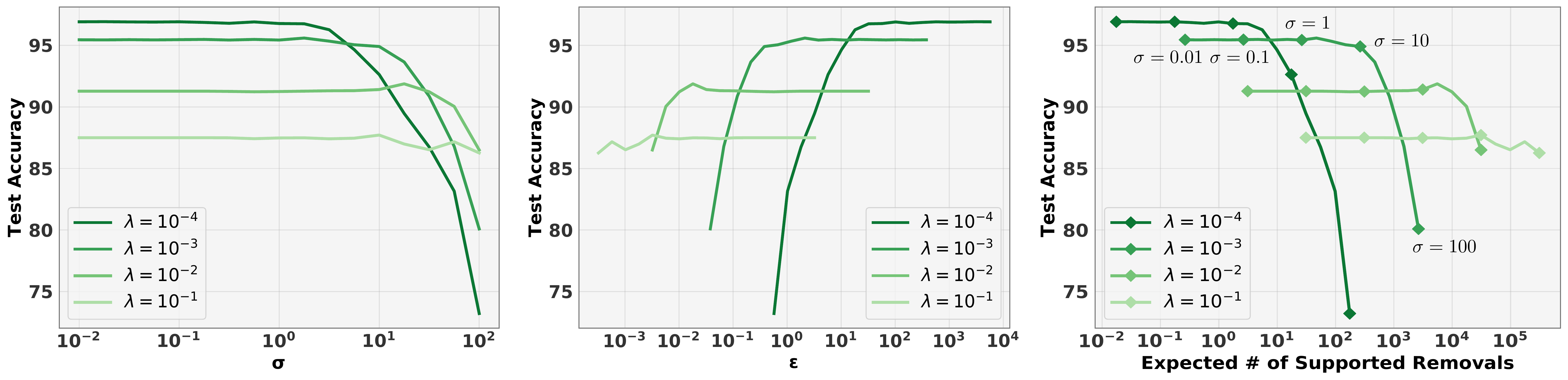}
\vspace{-5ex}
\caption{\textbf{Linear logistic regression on MNIST.} \textbf{Left}: Effect of $L_2$-regularization parameter, $\lambda$, and standard deviation of the objective perturbation, $\sigma$, on test accuracy. \textbf{Middle}: Effect of $\epsilon$ on test accuracy when supporting 100 removals. \textbf{Right}: Trade-off between accuracy and supported number of removals at $\epsilon=1$. At a given $\epsilon$, higher $\lambda$ and $\sigma$ values reduce test accuracy but allow for many more removals.} %Note that $\sigma$ and the expected supported number of removals are shown in log scale. }
\label{fig:mnist_std_lam}
\vspace{-1ex}
\end{figure*}

\subsection{Non-Linear Models}
Deep learning models often apply a linear
model to features extracted by a network pre-trained on a public dataset like ImageNet \citep{ren2015faster, he2017mask, zhao2017pyramid, carreira2017quovadis} for vision tasks, or from language model trained on public  text corpora \citep{devlin2019bert,dai2019transformer,yang2019xlnet,liu2019roberta} for natural language tasks. In such setups, we only need to worry about data removal from the linear model that is applied to the output of the feature extractor.

When feature extractors are trained on private data as well, we can use our certified-removal mechanism on linear models that are applied to the output of a differentially-private feature extraction network \citep{abadi2016deep}.\\% to obtain a certified-removal mechanism with better performance than pure differential privacy.

\begin{theorem}
Suppose $\Phi$ is a randomized learning algorithm that is $(\epsilon_\text{DP}, \delta_\text{DP})$-differentially private, and the outputs of $\Phi$ are used in a linear model by minimizing $L_\bb$ and using a removal mechanism that guarantees $(\epsilon_\text{CR},\delta_\text{CR})$-certified removal. Then the entire procedure guarantees
$(\epsilon_\text{DP} \!+\! \epsilon_\text{CR}, \delta_\text{DP} \!+\! \delta_\text{CR})$-certified removal.
\label{thm:extractor}
\end{theorem}

The advantage of this approach over training the entire network in a differentially private manner \citep{abadi2016deep} is that the (removal-enabled) linear model can be trained using a much smaller perturbation, which may greatly boost the accuracy of the final model (see Section \ref{sec:non-linear}).

\section{Experiments}
\label{sec:exp}
We test our certified removal mechanism in three settings: (1) removal from a standard linear logistic regressor, (2) removal from a linear logistic regressor that uses a feature extractor pre-trained on public data, and (3) removal from a non-linear logistic regressor by using a differentially private feature extractor. Code reproducing the results of our experiments is publicly available from \url{https://github.com/facebookresearch/certified-removal}. Table \ref{tab:timings} summarizes the training and removal times measured in our experiments.

\subsection{Linear Logistic Regression}
\label{sec:linear}

We first experiment on the MNIST digit classification dataset. For simplicity, we restrict to the binary classification problem of distinguishing between digits 3 and 8, and train a regularized logistic regressor using Algorithm \ref{alg:training}. Removal is performed using Algorithm \ref{alg:removal} with $\delta=\texttt{1e-4}$. %The output of training is a parameter vector $\bw \in \mathbb{R}^d$ with $d = 784$.

\paragraph{Effects of $\lambda$ and $\sigma$.} Training a removal-enabled model using Algorithm \ref{alg:training} requires selecting two hyperparameters: the $L_2$-regularization parameter, $\lambda$, and the standard deviation, $\sigma$, of the sampled perturbation vector $\bb$. Figure \ref{fig:mnist_std_lam} shows the effect of $\lambda$ and $\sigma$ on test accuracy and the expected number of removals supported before re-training. When fixing the supported number of removals at 100 (middle plot), the value of $\sigma$ is inversely related to $\epsilon$ (\emph{cf.} line 16 of Algorithm \ref{alg:removal}), hence higher $\epsilon$ results in smaller $\sigma$ and improved accuracy. Increasing $\lambda$ enables more removals before re-training (left and right plots) because it reduces the gradient residual norm, but very high values of $\lambda$ negatively affect test accuracy because the regularization term dominates the loss.

\paragraph{Tightness of the gradient residual norm bounds.} In Algorithm \ref{alg:removal} , we use the data-dependent bound from Corollaries \ref{cor:data_dependent} and \ref{cor:data_dependent_batch} to compute a \emph{per-data} or \emph{per-batch} estimate of the removal error, as opposed to the \emph{worst-case} bound in Theorems \ref{thm:asymptotic_bound} and \ref{thm:asymptotic_bound_batch}. Figure \ref{fig:mnist_tightness} shows the value of different bounds as a function of the number of removed points. We consider two removal scenarios: single point removal and batch removal with batch size $m=10$. We observe three phenomena: (1) The worst-case bounds (light blue and light green) are \emph{several orders of magnitude} higher than the data-dependent bounds (dark blue and dark green), which means that the number of supported removals is \emph{several orders of magnitude} higher when using the data-dependent bounds. (2) The cumulative sum of the gradient residual norm bounds is approximately linear for both the single and batch removal data-dependent bounds. (3) There remains a large gap between the data-dependent norm bounds and the true value of the gradient residual norm (dashed line), which suggests that the utility of our removal mechanism may be further improved via tighter analysis.

\begin{figure}[t]
\centering
\includegraphics[width=\columnwidth]{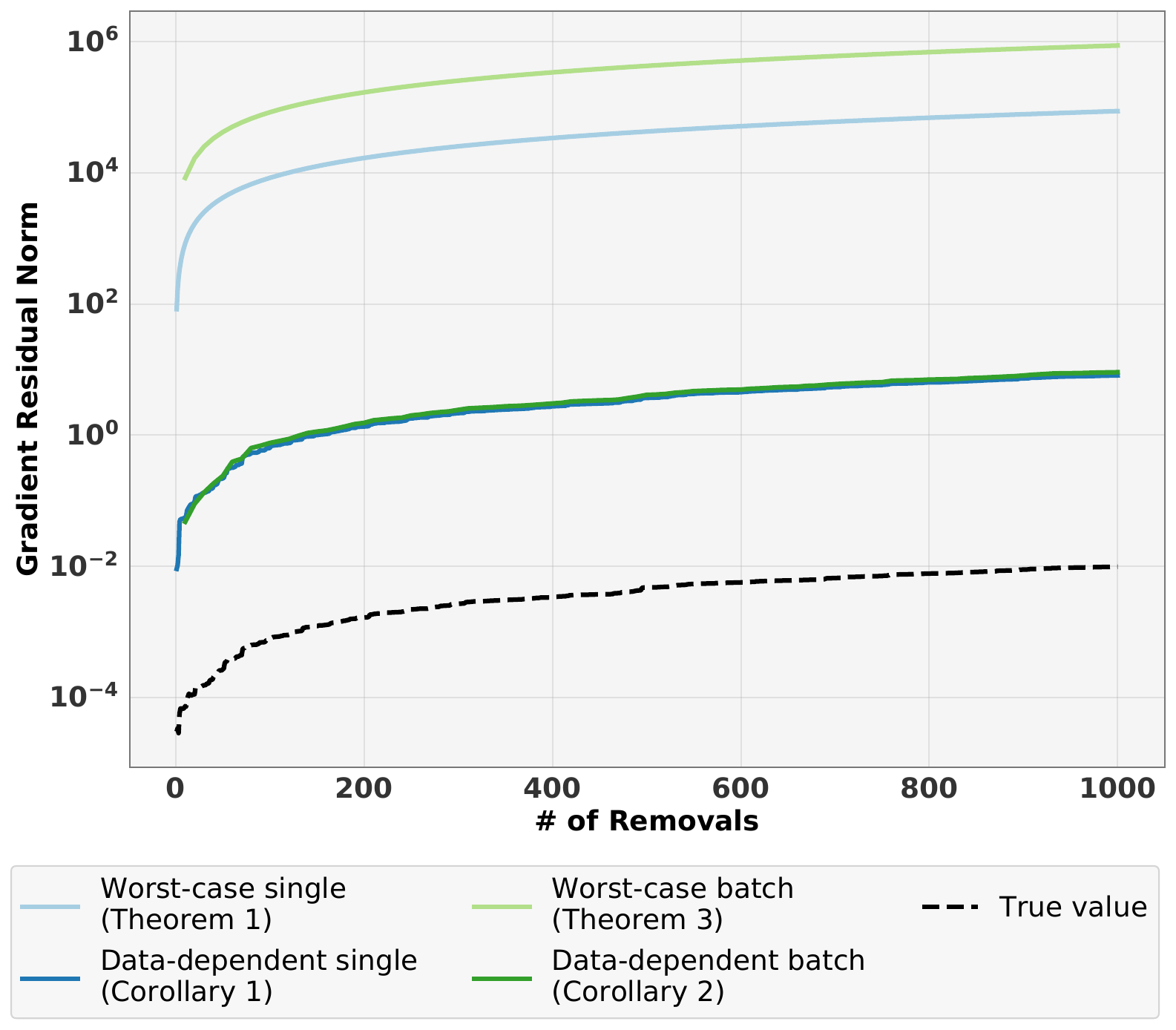}
\vspace{-4ex}
    \caption{\textbf{Linear logistic regression on MNIST.} Gradient residual norm (on log scale) as a function of the number of removals.}
\label{fig:mnist_tightness}
\vspace{-2ex}
\end{figure}

\begin{figure*}[t]
\centering
\includegraphics[width=\textwidth]{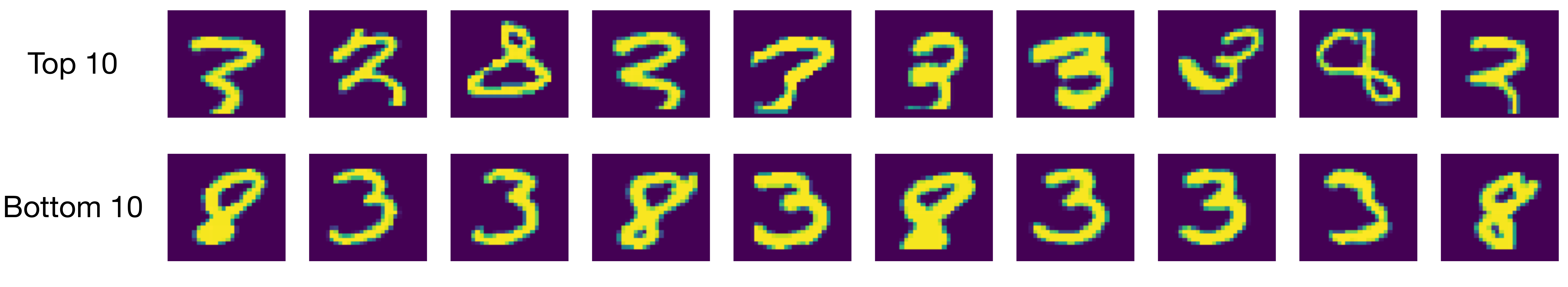}
\vspace{-5ex}
    \caption{\textbf{MNIST training digits sorted by norm of the removal update $\mathbf{\| H_{\bw^*}^{-1} \Delta \|_2}$.} The samples with the highest norm (\textbf{top}) appear to be atypical, making it harder to undo their effect on the model. The samples with the lowest norm (\textbf{bottom}) are prototypical 3s and 8s, and hence are much easier to remove.}
\label{fig:mnist_grad_norm}
\end{figure*}

\paragraph{Gradient residual norm and removal difficulty.} The data-dependent bound is governed by the norm of the update $H_{\bw^*}^{-1} \Delta$, which measures the influence of the removed point on the parameters and varies greatly depending on the training sample being removed. %We offer an intuitive interpretation for this data-dependent value in terms of the difficulty of the removed sample.
Figure \ref{fig:mnist_grad_norm} shows the training samples corresponding to the 10 largest and smallest values of $\| H_{\bw^*}^{-1} \Delta \|_2$. There are large visual differences between these samples: large values correspond to oddly-shaped 3s and 8s, while small values correspond to ``prototypical'' digits. This suggests that removing outliers is harder, because the model tends to memorize their details and their impact on the model is easy to distinguish from other samples.

\begin{figure*}[t]
\begin{minipage}{.63\textwidth}
\centering
\includegraphics[width=\textwidth]{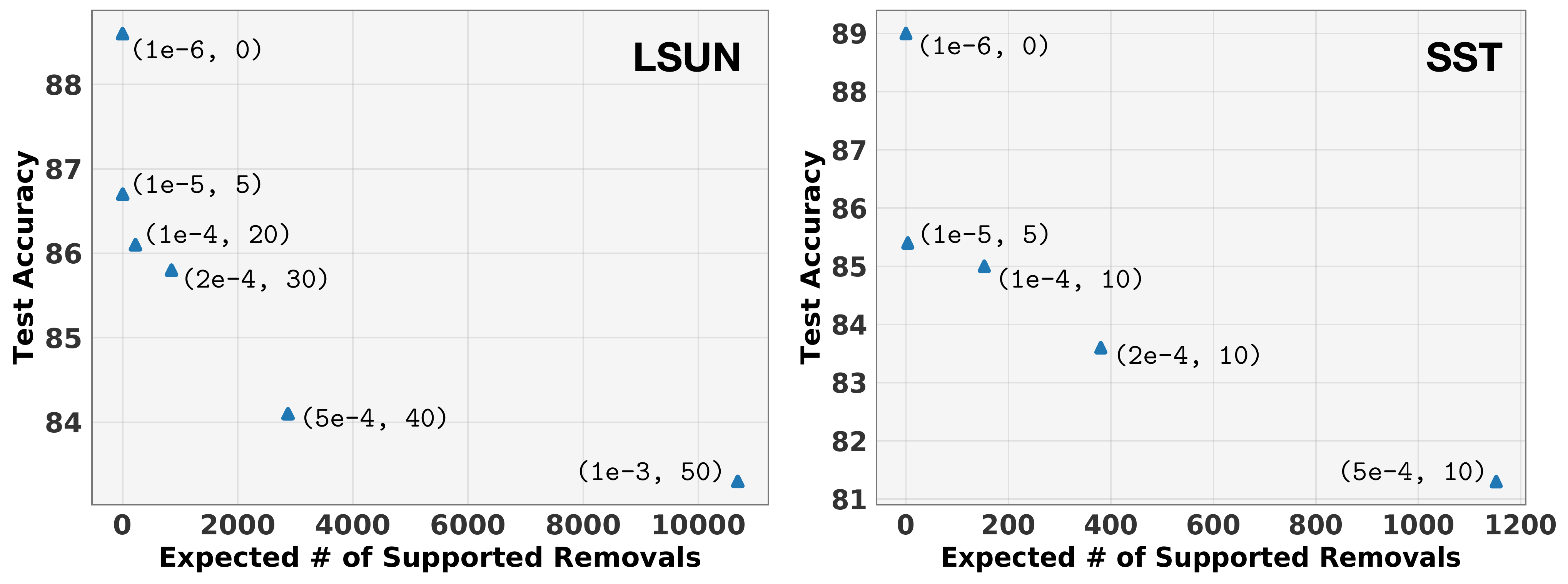}
\vspace{-5ex}
\caption{\textbf{Linear models trained on public feature extractors.} Trade-off between test accuracy and the expected number of supported removals (at $\epsilon\!=\!1$) on LSUN (\textbf{left}) and SST (\textbf{right}). The setting of $(\lambda, \sigma)$ is shown next to each point. The number of supported removals rapidly increases when accuracy is slightly sacrificed.}
\label{fig:lsun_sst_tradeoff}
\end{minipage}
\hspace{2ex}
\begin{minipage}{.33\textwidth}
\centering
\includegraphics[width=\columnwidth]{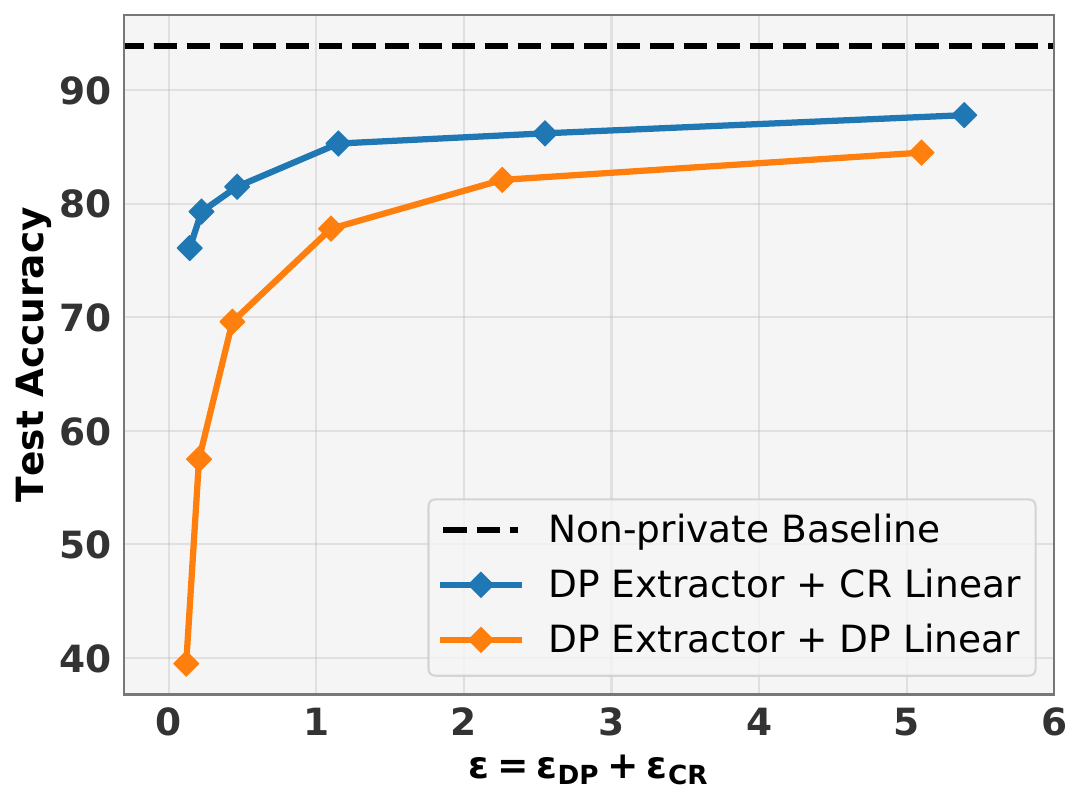}
\vspace{-5ex}
\caption{\textbf{Using $\mathbf{\epsilon}$-DP features.} Trade-off between $\epsilon$ and test accuracy on SVHN of models that support 10 removals. Dashed line shows non-private model accuracy.} % For reference, the accuracy of a model that does not support removal is also shown (dashed line).}
\label{fig:svhn_tradeoff}
\end{minipage}
\vspace{-2ex}
\end{figure*}

\subsection{Non-Linear Logistic Regression using Public,~Pre-Trained Feature Extractors}
\label{sec:text_and_image}
%\laurens{We need to make much more explicit here that we adapt the learning setting to assume the availability of a large, public dataset for pre-training from which we never have to remove any data.}

We consider the common scenario in which a feature extractor is trained on public data (\emph{i.e.}, does not require removal), and a linear classifier is trained on these features using non-public data. We study two tasks: (1) scene classification on the LSUN dataset and (2) sentiment classification on the Stanford Sentiment Treebank (SST) dataset. We subsample the LSUN dataset to 100K images per class (\emph{i.e.}, $n = 1\textrm{M}$).

%For these complex datasets, linear models such as logistic regression and SVM do not attain competitive test accuracy. However, most modern techniques for both language and vision tasks involve the use of a feature extractor pre-trained on a large \emph{public} dataset. Linear models that are trained on top of such feature extractors often attain close to state-of-the-art results.
For LSUN, we extract features using a ResNeXt-101 model \citep{xie2017resnext} trained on 1B Instagram images \citep{mahajan2018exploring} and fine-tuned on ImageNet \citep{deng2009imagenet}. For SST, we extract features using a pre-trained RoBERTa \citep{liu2019roberta} language model. At removal time, we use Algorithm \ref{alg:removal} with $\epsilon=1$ and $\delta=\texttt{1e-4}$ in both experiments.

\paragraph{Result on LSUN.} We reduce the 10-way LSUN classification task to 10 one-versus-all tasks and randomly subsample the negative examples to ensure the positive and negative classes are balanced in all binary classification problems. Subsampling benefits  removal since a training sample does not always need to be removed from all 10 classifiers.

Figure \ref{fig:lsun_sst_tradeoff} (left) shows the relationship between test accuracy and the expected number of removals on LSUN. The value of $(\lambda, \sigma)$ is shown next to each point, with the left-most point corresponding to training a regular model that supports no removal. At the cost of a small drop in accuracy (from $88.6\%$ to $83.3\%$), the model supports over $10,000$ removals before re-training is needed. As shown in Table \ref{tab:timings}, the computational cost for removal is more than $250 \times$ smaller than re-training the model on the remaining data points.

\paragraph{Result on SST.} SST is a sentiment classification dataset commonly used for benchmarking language models \citep{wang2019glue}. We use SST in the binary classification task of predicting whether or not a movie review is positive. Figure \ref{fig:lsun_sst_tradeoff} (right) shows the trade-off between accuracy and supported number of removals. The regular model (left-most point) attains a test accuracy of $89.0\%$, which matches the performance of competitive prior work \citep{tai2015improved, wieting2016towards, looks2017deep}. As before, a large number of removals is supported at a small loss in test accuracy; the computational costs for removal are $870\times$ lower than for re-training the model.

\subsection{Non-linear Logistic Regression using Differentially~Private Feature Extractors}
\label{sec:non-linear}

When public data is not available for training a feature extractor, we can train a differentially private feature extractor on private data \citep{abadi2016deep} and apply Theorem \ref{thm:extractor} to remove data from the final (removal-enabled) linear layer. This approach has a major advantage over training the entire model using the approach of \citep{abadi2016deep} because the final linear layer can partly correct for the noisy features produced by the private feature extractor.

We evaluate this approach on the Street View House Numbers (SVHN) digit classification dataset. We compare it to a differentially private CNN\footnote{We use a simple CNN with two convolutional layers with 64 filters of size $3 \times 3$ and $2 \times 2$ max-pooling.} trained using the technique of \citet{abadi2016deep}. Since the CNN is differentially private, certified removal is achieved trivially without applying any removal. For a fair comparison, we fix $\delta = \texttt{1e-4}$ and train $(\epsilon_\text{DP}/10, \delta)$-private CNNs for a range of values of $\epsilon_\text{DP}$. By the union bound for group privacy \citep{dwork2011differential}, the resulting models support up to then $(\epsilon_\text{DP}, \delta)$-CR removals.

To measure the effectiveness of Theorem \ref{thm:extractor} for certified data removal, we also train an $(\epsilon_\text{DP}/10, \delta/10)$-differentially private CNN and extract features from its penultimate linear. We use these features in Algorithm \ref{alg:training} to train 10 one-versus-all classifiers with total failure probability of at most $\frac{9}{10} \delta$. Akin to the experiments on LSUN, we subsample the negative examples in each of the binary classifiers to speed up removal. The expected contribution to $\epsilon$ from the updates is set to $\epsilon_\text{CR} \approx \epsilon_\text{DP}/10$, hence achieving $(\epsilon, \delta)$-CR with $\epsilon = \epsilon_\text{DP} + \epsilon_\text{CR} \approx \epsilon_\text{DP} + \epsilon_\text{DP} / 10$ after 10 removals.

Figure \ref{fig:svhn_tradeoff} shows the relationship between test accuracy and $\epsilon$ for both the fully private and the Newton update removal methods. For reference, the dashed line shows the accuracy obtained by a non-private CNN that does not support removal. For smaller values or $\epsilon$, training a private feature extractor (blue) and training the linear layer using Algorithm \ref{alg:training} attains much higher test accuracy than training a fully differentially private model (orange). In particular, at $\epsilon \approx 0.1$, the fully differentially private baseline's accuracy is only $22.7\%$, whereas our approach attains a test accuracy of $71.2\%$. Removal from the linear model trained on top of the private extractor only takes 0.27s, compared to more than 1.5 hour when re-training the CNN from scratch.

\section{Related Work}
Removal of specific training samples from models has been studied in prior work on decremental learning \citep{cauwenberghs2000incremental, karasuyama2009multiple, tsai2014incremental} and machine unlearning \citep{cao2015towards}. \citet{ginart2019making} studied the problem of removing data from $k$-means clusterings. These studies aim at \emph{exact removal} of one or more training samples from a trained model: their success measure is closeness to the optimal parameter or objective value. This suffices for purposes such as quickly evaluating the leave-one-out error or correcting mislabeled data, but it does not provide a formal guarantee of statistical indistinguishability. Our work leverages differential privacy to develop a more rigorous definition of data removal. Concurrent work \citep{bourtoule2019machine} presents an approach that allows certified removal with $\epsilon=0$.

Our definition of certified removal uses the same notion of indistinguishability as that of differential privacy. Many classical machine learning algorithms have been shown to support differentially private versions, including PCA \citep{chaudhuri2012pca}, matrix factorization \citep{liu2015fast}, linear models \citep{chaudhuri2011dperm}, and neural networks \citep{abadi2016deep}. Our removal mechanism can be viewed on a spectrum of noise addition techniques for preserving data privacy, balancing between computation time for the removal mechanism and model utility. We hope to further explore the connections between differential privacy and certified removal in follow-up work to design certified-removal algorithms with better guarantees and computational efficiency.

\section{Conclusion}
We have studied a mechanism that quickly ``removes'' data from a machine-learning model up to a differentially private guarantee: the model after removal is indistinguishable from a model that never saw the removed data to begin with. While we demonstrate that this mechanism is practical in some settings, at least four challenges for future work remain. (1) The Newton update removal mechanism requires inverting the Hessian matrix, which may be problematic. Methods that approximate the Hessian with near-diagonal matrices may address this problem. (2) Removal from models with non-convex losses is unsupported; it may require local analysis of the loss surface to show that data points do not move the model out of a local optimum. (3) There remains a large gap between our data-dependent bound and the true gradient residual norm, necessitating a tighter analysis. (4) Some applications may require the development of alternative, less constraining notions of data removal.

\newpage
\bibliography{citations}

\begin{thebibliography}{}

\bibitem[Abadi et~al., 2016]{abadi2016deep}
Abadi, M., Chu, A., Goodfellow, I.~J., McMahan, H.~B., Mironov, I., Talwar, K.,
  and Zhang, L. (2016).
\newblock Deep learning with differential privacy.
\newblock In {\em Proceedings of the 2016 {ACM} {SIGSAC} Conference on Computer
  and Communications Security, Vienna, Austria, October 24-28, 2016}, pages
  308--318.

\bibitem[Biggio et~al., 2012]{biggio2012}
Biggio, B., Nelson, B., and Laskov, P. (2012).
\newblock Poisoning attacks against support vector machines.
\newblock pages 1467--1474.

\bibitem[Bourtoule et~al., 2019]{bourtoule2019machine}
Bourtoule, L., Chandrasekaran, V., Choquette-Choo, C., Jia, H., Travers, A.,
  Zhang, B., Lie, D., and Papernot, N. (2019).
\newblock Machine unlearning.
\newblock In {\em arXiv 1912.03817}.

\bibitem[Cao and Yang, 2015]{cao2015towards}
Cao, Y. and Yang, J. (2015).
\newblock Towards making systems forget with machine unlearning.
\newblock In {\em 2015 {IEEE} Symposium on Security and Privacy, {SP} 2015, San
  Jose, CA, USA, May 17-21, 2015}, pages 463--480.

\bibitem[Carlini et~al., 2019]{carlini2019secret}
Carlini, N., Liu, C., Kos, J., Erlingsson, U., and Song, D. (2019).
\newblock The secret sharer: Measuring unintended neural network memorization
  \& extracting secrets.
\newblock In {\em USENIX Security Symposium}, pages 267--284.

\bibitem[Carreira and Zisserman, 2017]{carreira2017quovadis}
Carreira, J. and Zisserman, A. (2017).
\newblock Quo vadis, action recognition? {A} new model and the kinetics
  dataset.
\newblock {\em CoRR}, abs/1705.07750.

\bibitem[Cauwenberghs and Poggio, 2000]{cauwenberghs2000incremental}
Cauwenberghs, G. and Poggio, T.~A. (2000).
\newblock Incremental and decremental support vector machine learning.
\newblock In {\em Advances in Neural Information Processing Systems 13, Papers
  from Neural Information Processing Systems {(NIPS)} 2000, Denver, CO, {USA}},
  pages 409--415.

\bibitem[Chaudhuri et~al., 2011]{chaudhuri2011dperm}
Chaudhuri, K., Monteleoni, C., and Sarwate, A.~D. (2011).
\newblock Differentially private empirical risk minimization.
\newblock {\em Journal of Machine Learning Research}, 12:1069--1109.

\bibitem[Chaudhuri et~al., 2012]{chaudhuri2012pca}
Chaudhuri, K., Sarwate, A.~D., and Sinha, K. (2012).
\newblock Near-optimal differentially private principal components.
\newblock In {\em Advances in Neural Information Processing Systems 25: 26th
  Annual Conference on Neural Information Processing Systems 2012. Proceedings
  of a meeting held December 3-6, 2012, Lake Tahoe, Nevada, United States.},
  pages 998--1006.

\bibitem[Cook and Weisberg, 1982]{cook1982residuals}
Cook, R.~D. and Weisberg, S. (1982).
\newblock {\em Residuals and influence in regression}.
\newblock New York: Chapman and Hall.

\bibitem[Dai et~al., 2019]{dai2019transformer}
Dai, Z., Yang, Z., Yang, Y., Carbonell, J.~G., Le, Q.~V., and Salakhutdinov, R.
  (2019).
\newblock Transformer-xl: Attentive language models beyond a fixed-length
  context.
\newblock In {\em Proceedings of the 57th Conference of the Association for
  Computational Linguistics, {ACL} 2019, Florence, Italy, July 28- August 2,
  2019, Volume 1: Long Papers}, pages 2978--2988.

\bibitem[Deng et~al., 2009]{deng2009imagenet}
Deng, J., Dong, W., Socher, R., Li, L., Li, K., and Li, F. (2009).
\newblock Imagenet: {A} large-scale hierarchical image database.
\newblock In {\em 2009 {IEEE} Computer Society Conference on Computer Vision
  and Pattern Recognition {(CVPR} 2009), 20-25 June 2009, Miami, Florida,
  {USA}}, pages 248--255.

\bibitem[Devlin et~al., 2019]{devlin2019bert}
Devlin, J., Chang, M., Lee, K., and Toutanova, K. (2019).
\newblock {BERT:} pre-training of deep bidirectional transformers for language
  understanding.
\newblock In {\em Proceedings of the 2019 Conference of the North American
  Chapter of the Association for Computational Linguistics: Human Language
  Technologies, {NAACL-HLT} 2019, Minneapolis, MN, USA, June 2-7, 2019, Volume
  1 (Long and Short Papers)}, pages 4171--4186.

\bibitem[Dwork, 2011]{dwork2011differential}
Dwork, C. (2011).
\newblock Differential privacy.
\newblock {\em Encyclopedia of Cryptography and Security}, pages 338--340.

\bibitem[Ginart et~al., 2019]{ginart2019making}
Ginart, A., Guan, M.~Y., Valiant, G., and Zou, J. (2019).
\newblock Making {AI} forget you: Data deletion in machine learning.
\newblock {\em CoRR}, abs/1907.05012.

\bibitem[Gu and Eisenstat, 1995]{gu1995downdating}
Gu, M. and Eisenstat, S.~C. (1995).
\newblock Downdating the singular value decomposition.
\newblock {\em SIAM J. Matrix Anal. Appl.}, 16(3):793--810.

\bibitem[He et~al., 2017]{he2017mask}
He, K., Gkioxari, G., Doll{\'{a}}r, P., and Girshick, R.~B. (2017).
\newblock Mask {R-CNN}.
\newblock In {\em {IEEE} International Conference on Computer Vision, {ICCV}
  2017, Venice, Italy, October 22-29, 2017}, pages 2980--2988.

\bibitem[Karasuyama and Takeuchi, 2009]{karasuyama2009multiple}
Karasuyama, M. and Takeuchi, I. (2009).
\newblock Multiple incremental decremental learning of support vector machines.
\newblock In {\em Advances in Neural Information Processing Systems 22: 23rd
  Annual Conference on Neural Information Processing Systems 2009. Proceedings
  of a meeting held 7-10 December 2009, Vancouver, British Columbia, Canada.},
  pages 907--915.

\bibitem[Koh and Liang, 2017]{koh2017understanding}
Koh, P.~W. and Liang, P. (2017).
\newblock Understanding black-box predictions via influence functions.
\newblock In {\em Proceedings of the 34th International Conference on Machine
  Learning, {ICML} 2017, Sydney, NSW, Australia, 6-11 August 2017}, pages
  1885--1894.

\bibitem[Liu and Nocedal, 1989]{liu1989lbfgs}
Liu, D.~C. and Nocedal, J. (1989).
\newblock On the limited memory bfgs method for large scale optimization.
\newblock {\em Math. Program.}, 45(1-3):503--528.

\bibitem[Liu et~al., 2019]{liu2019roberta}
Liu, Y., Ott, M., Goyal, N., Du, J., Joshi, M., Chen, D., Levy, O., Lewis, M.,
  Zettlemoyer, L., and Stoyanov, V. (2019).
\newblock Roberta: {A} robustly optimized {BERT} pretraining approach.
\newblock {\em CoRR}, abs/1907.11692.

\bibitem[Liu et~al., 2015]{liu2015fast}
Liu, Z., Wang, Y.-X., and Smola, A. (2015).
\newblock Fast differentially private matrix factorization.
\newblock In {\em Proceedings of the 9th ACM Conference on Recommender
  Systems}, pages 171--178. ACM.

\bibitem[Looks et~al., 2017]{looks2017deep}
Looks, M., Herreshoff, M., Hutchins, D., and Norvig, P. (2017).
\newblock Deep learning with dynamic computation graphs.
\newblock In {\em 5th International Conference on Learning Representations,
  {ICLR} 2017, Toulon, France, April 24-26, 2017, Conference Track
  Proceedings}.

\bibitem[Mahajan et~al., 2018]{mahajan2018exploring}
Mahajan, D., Girshick, R.~B., Ramanathan, V., He, K., Paluri, M., Li, Y.,
  Bharambe, A., and van~der Maaten, L. (2018).
\newblock Exploring the limits of weakly supervised pretraining.
\newblock In {\em Computer Vision - {ECCV} 2018 - 15th European Conference,
  Munich, Germany, September 8-14, 2018, Proceedings, Part {II}}, pages
  185--201.

\bibitem[Ren et~al., 2015]{ren2015faster}
Ren, S., He, K., Girshick, R.~B., and Sun, J. (2015).
\newblock Faster {R-CNN:} towards real-time object detection with region
  proposal networks.
\newblock In {\em Advances in Neural Information Processing Systems 28: Annual
  Conference on Neural Information Processing Systems 2015, December 7-12,
  2015, Montreal, Quebec, Canada}, pages 91--99.

\bibitem[Tai et~al., 2015]{tai2015improved}
Tai, K.~S., Socher, R., and Manning, C.~D. (2015).
\newblock Improved semantic representations from tree-structured long
  short-term memory networks.
\newblock In {\em Proceedings of the 53rd Annual Meeting of the Association for
  Computational Linguistics and the 7th International Joint Conference on
  Natural Language Processing of the Asian Federation of Natural Language
  Processing, {ACL} 2015, July 26-31, 2015, Beijing, China, Volume 1: Long
  Papers}, pages 1556--1566.

\bibitem[Thomee et~al., 2016]{thomee2016yfcc100m}
Thomee, B., Shamma, D.~A., Friedland, G., Elizalde, B., Ni, K., Poland, D.,
  Borth, D., and Li, L.-J. (2016).
\newblock Yfcc100m: The new data in multimedia research.
\newblock {\em Communications of the ACM}, 59(2):64--73.

\bibitem[Tsai et~al., 2014]{tsai2014incremental}
Tsai, C., Lin, C., and Lin, C. (2014).
\newblock Incremental and decremental training for linear classification.
\newblock In {\em The 20th {ACM} {SIGKDD} International Conference on Knowledge
  Discovery and Data Mining, {KDD} '14, New York, NY, {USA} - August 24 - 27,
  2014}, pages 343--352.

\bibitem[Wang et~al., 2019]{wang2019glue}
Wang, A., Singh, A., Michael, J., Hill, F., Levy, O., and Bowman, S.~R. (2019).
\newblock {GLUE:} {A} multi-task benchmark and analysis platform for natural
  language understanding.
\newblock In {\em 7th International Conference on Learning Representations,
  {ICLR} 2019, New Orleans, LA, USA, May 6-9, 2019}.

\bibitem[Wieting et~al., 2016]{wieting2016towards}
Wieting, J., Bansal, M., Gimpel, K., and Livescu, K. (2016).
\newblock Towards universal paraphrastic sentence embeddings.
\newblock In {\em 4th International Conference on Learning Representations,
  {ICLR} 2016, San Juan, Puerto Rico, May 2-4, 2016, Conference Track
  Proceedings}.

\bibitem[Xie et~al., 2017]{xie2017resnext}
Xie, S., Girshick, R.~B., Doll{\'{a}}r, P., Tu, Z., and He, K. (2017).
\newblock Aggregated residual transformations for deep neural networks.
\newblock In {\em 2017 {IEEE} Conference on Computer Vision and Pattern
  Recognition, {CVPR} 2017, Honolulu, HI, USA, July 21-26, 2017}, pages
  5987--5995.

\bibitem[Yang et~al., 2019]{yang2019xlnet}
Yang, Z., Dai, Z., Yang, Y., Carbonell, J.~G., Salakhutdinov, R., and Le, Q.~V.
  (2019).
\newblock Xlnet: Generalized autoregressive pretraining for language
  understanding.
\newblock {\em CoRR}, abs/1906.08237.

\bibitem[Yeom et~al., 2018]{yeom2018privacy}
Yeom, S., Giacomelli, I., Fredrikson, M., and Jha, S. (2018).
\newblock Privacy risk in machine learning: Analyzing the connection to
  overfitting.
\newblock In {\em CSF}.

\bibitem[Zhao et~al., 2017]{zhao2017pyramid}
Zhao, H., Shi, J., Qi, X., Wang, X., and Jia, J. (2017).
\newblock Pyramid scene parsing network.
\newblock In {\em 2017 {IEEE} Conference on Computer Vision and Pattern
  Recognition, {CVPR} 2017, Honolulu, HI, USA, July 21-26, 2017}, pages
  6230--6239.

\end{thebibliography}
\bibliographystyle{apalike}

\newpage
\appendix
\onecolumn
\setcounter{theorem}{0}

\section{Appendix}

We present proofs for theorems stated in the main paper.\\

\begin{theorem}
Suppose that $\forall (\bx_i, y_i) \in \calD, \bw \in \mathbb{R}^d: \|\nabla \ell(\bw^\top \bx_i, y_i)\|_2 \leq C$. Suppose also that $\ell''$ is $\gamma$-Lipschitz and $\| \bx_i \|_2 \leq 1$ for all $(\bx_i, y_i) \in \calD$. Then:
\begin{align*}
\| \nabla L(\bw^-; \calD') \|_2 &= \|(H_{\bw_\eta} - H_{\bw^*}) H_{\bw^*}^{-1} \Delta\|_2 \\
&\leq \gamma (n-1) \| H_{\bw^*}^{-1} \Delta \|_2^2 \leq \frac{4 \gamma C^2}{\lambda^2 (n-1)},
\end{align*}
where $H_{\bw_\eta}$ denotes the Hessian of $L(\cdot; \calD')$ at the parameter vector $\bw_\eta = \bw^* + \eta H_{\bw^*}^{-1} \Delta$ for some $\eta \in [0,1]$.
\end{theorem}

\begin{proof}
Let $G(\bw) = \nabla L(\bw; \calD')$ denote the gradient at $\bw$ of the empirical risk on the reduced dataset $\calD'$. Note that $G : \mathbb{R}^d \rightarrow \mathbb{R}^d$ is a vector-valued function. By Taylor's Theorem, there exists some $\eta \in [0,1]$ such that:
\begin{align*}
G(\bw^-) &= G(\bw^* + H_{\bw^*}^{-1} \Delta) \\
&= G(\bw^*) + \nabla G(\bw^* + \eta H_{\bw^*}^{-1} \Delta) H_{\bw^*}^{-1} \Delta.
\end{align*}
Since $G$ is the gradient of $L(\cdot; \calD')$, the quantity $\nabla G(\bw^* + \eta H_{\bw^*}^{-1} \Delta)$ is exactly the Hessian of $L(\cdot; \calD')$ evaluated at the point $\bw_\eta = \bw^* + \eta H_{\bw^*}^{-1} \Delta$. Thus:
\begin{align*}
G(\bw^-) &= G(\bw^*) + H_{\bw_\eta} H_{\bw^*}^{-1} \Delta \\
&= \left(G(\bw^*) + \Delta\right) + H_{\bw_\eta} H_{\bw^*}^{-1} \Delta - \Delta \\
&= 0 + H_{\bw_\eta} H_{\bw^*}^{-1} \Delta - H_{\bw^*} H_{\bw^*}^{-1} \Delta \\
&= (H_{\bw_\eta} - H_{\bw^*}) H_{\bw^*}^{-1} \Delta.
\end{align*}
This gives:
\begin{align*}
\| G(\bw^-) \|_2 &= \| (H_{\bw_\eta} - H_{\bw^*}) H_{\bw^*}^{-1} \Delta \|_2 \\
&\leq \| H_{\bw_\eta} - H_{\bw^*} \|_2 \| H_{\bw^*}^{-1} \Delta \|_2.
\end{align*}
Using the Lipschitz-ness of $\ell''$, we have for every $i$:
\begin{align*}
\| \nabla^2 \ell(\bw_\eta^\top \bx_i, y_i) - \nabla_\bw^2 \ell((\bw^*)^\top \bx_i, y_i)\|_2 &= \| [\ell''(\bw_\eta^\top \bx_i, y_i) - \ell''((\bw^*)^\top \bx_i, y_i)] \bx_i \bx_i^\top \|_2 \\
&\leq | \ell''(\bw_\eta^\top \bx_i, y_i) - \ell''((\bw^*)^\top \bx_i, y_i) | \cdot \| \bx_i \|_2^2 \\
&\leq \gamma \| \bw_\eta - \bw^* \|_2 \hspace{4ex} \text{\small since $\| \bx_i \|_2 \leq 1$} \\
&= \gamma \| \eta H_{\bw^*}^{-1} \Delta \|_2 \\
&\leq \gamma \| H_{\bw^*}^{-1} \Delta \|_2.
\end{align*}
As a result, we can conclude that:
\begin{align*}
\| H_{\bw_\eta} - H_{\bw^*} \|_2 &\leq \sum_{i=1}^{n-1} \Big\Vert \nabla^2 \ell(\bw_\eta^\top \bx_i, y_i) - \nabla^2 \ell((\bw^*)^\top \bx_i, y_i) \Big\Vert_2 \\
&\leq \gamma (n-1) \| H_{\bw^*}^{-1} \Delta \|_2.
\end{align*}
Combining these results leads us to conclude that $\| G(\bw^-) \|_2 \leq \gamma (n-1) \| H_{\bw^*}^{-1} \Delta \|_2^2$.

We can simplify this bound by analyzing $\| H_{\bw^*}^{-1} \Delta \|_2$.
Since $L(\cdot; \calD')$ is $\lambda (n-1)$-strongly convex, we get $\| H_{\bw^*} \|_2 \geq \lambda (n-1)$, hence $\| H_{\bw^*}^{-1} \|_2 \leq \frac{1}{\lambda (n-1)}$. Recall that
$$\Delta = \lambda \bw^* + \nabla \ell\left((\bw^*)^\top \bx_n, y_n\right).$$ Since $\bw^*$ is the global optimal solution of the loss $L(\cdot; \calD)$, we obtain the condition:
$$0 = \nabla L(\bw^*; \calD) = \sum_{i=1}^n \nabla \ell\left((\bw^*)^\top \bx_i, y_i\right) + \lambda n \bw^*.$$
Using the norm bound $\| \nabla \ell(\bw^\top \bx, y) \|_2 \leq C$ and re-arranging the terms, we obtain:
$$\| \bw^* \|_2 = \frac{\| \sum_{i=1}^n \nabla \ell((\bw^*)^\top \bx_i, y_i) \|_2}{\lambda n} \leq \frac{C}{\lambda}.$$
Using this and the same norm bound, we observe:
$$\| \Delta \|_2 \leq \lambda \| \bw^* \|_2 + \| \nabla \ell((\bw^*)^\top \bx_n, y_n) \|_2 \leq 2C,$$ from which we obtain:
$$\| H_{\bw^*}^{-1} \Delta \|_2 \leq \| H_{\bw^*}^{-1} \|_2 \| \Delta \|_2 \leq \frac{2C}{\lambda (n-1)},$$
which leads to the desired bound.
\end{proof}

\begin{theorem}
Suppose that $\bb$ is drawn from a distribution with density function $p(\cdot)$
such that for any $\bb_1, \bb_2 \in \mathbb{R}^d$ satisfying
$\|\bb_1 - \bb_2\|_2 \leq \epsilon'$, we have that:
$e^{-\epsilon} \leq \frac{p(\bb_1)}{p(\bb_2)} \leq e^{\epsilon}$.
Then:
$$e^{-\epsilon} \leq \frac{f_{\tilde{A}}(\tilde{\bw})}{f_A(\tilde{\bw})} \leq e^\epsilon,$$
for any solution $\tilde{\bw}$ produced by $\tilde{A}$.
\end{theorem}

\begin{proof}
Let $p$ be the density function of $\bb$ and let $g_{\tilde{A}}$ be the density
functions of the gradient residual under optimizer $\tilde{A}$. Consider the density functions
$q_A$ and $q_{\tilde{A}}$ of $\bz = \bb - \bu$ under optimizers $A$ and $\tilde{A}$.
We obtain:
\begin{align*}
q_{\tilde{A}}(\bz) &= \int_\bv g_{\tilde{A}}(\bv) p(\bz + \bv) d\bv \\
&= \int_{\bv : \|\bv\|_2 \leq \epsilon'} g_{\tilde{A}}(\bv) p(\bz + \bv) d\bv \hspace{1ex} \text{\small{since $g_{\tilde{A}}$ has no support elsewhere}} \\
&\leq \int_{\bv : \|\bv\|_2 \leq \epsilon'} g_{\tilde{A}}(\bv) e^\epsilon p(\bz) d\bv \hspace{3ex} \text{\small{since $\|\bv\|_2 \leq \epsilon'$}} \\
&= e^\epsilon p(\bz) \\
&= e^\epsilon q_A(\bz)
\end{align*}
where the last step follows since the gradient residual $\bu$ under $A$ is 0.
To complete the proof, note that the value of $\tilde{\bw}$ is completely determined
by $\bz \!=\! \bb \!-\! \bu$. Indeed, any $\bw$ satisfying Equation \ref{eq:approximate_gradient}
is an exact solution of the strongly convex loss $L_{\bb}({\bw}) - \bu^\top \bw$ and, hence, must be unique. This gives:
\begin{align*}
f_{\tilde{A}}(\tilde{\bw}) &= \int_\bz f_{\tilde{A}}(\tilde{\bw} | \bz) q_{\tilde{A}}(\bz) d\bz \\
&= \int_\bz f_A(\tilde{\bw} | \bz) q_{\tilde{A}}(\bz) d\bz \hspace{1ex} \text{\small{since $\tilde{\bw}$ is governed by $\bz$}} \\
&\leq \int_\bz f_A(\tilde{\bw} | \bz) e^\epsilon q_A(\bz) d\bz \\
&= e^\epsilon f_A(\tilde{\bw}).
\end{align*}
In the above, note that while $f_{\tilde{A}}$ and $f_A$ are not the same in general,  their difference is governed entirely by $\bz$: given a fixed $\bz$, the \emph{conditional} density of $\tilde{\bw}$ is the same under both density functions.

Using a similar approach as above, we can also show that $f_{\tilde{A}}(\tilde{\bw}) \geq e^{-\epsilon} f_A(\tilde{\bw})$.
\end{proof}

\begin{theorem}
Let $A$ be the learning algorithm that returns the unique optimum of the loss $L_{\bb}(\bw; \calD)$ and let $M$ be the Newton update removal mechanism (cf., Equation~\ref{eq:newton_update}). Suppose that $\| \nabla L(\bw^-; \calD') \|_2 \leq \epsilon'$ for some computable bound $\epsilon' > 0$. We have the following guarantees for $M$:
\begin{enumerate}[(i)]
\setlength\itemsep{0ex}
\vspace{-1ex}
\item If $\bb$ is drawn from a distribution with density $p(\bb) \propto e^{-\frac{\epsilon}{\epsilon'} \|\bb\|_2}$, then $M$ is $\epsilon$-CR for $A$;
\item If $\bb \sim \mathcal{N}(0, c \epsilon' / \epsilon)^d$ with $c > 0$, then $M$ is $(\epsilon,\delta)$-CR for $A$ with $\delta = 1.5 \cdot e^{-c^2/2}$.
\end{enumerate}
\end{theorem}

\begin{proof}
The proof involves bounding the density ratio of $\bb_1$ and $\bb_2$ when $\| \bb_1 - \bb_2 \|_2 \leq \epsilon'$ and then invoking Theorem \ref{thm:epsilon_cd}.
\begin{enumerate}[(i)]
\setlength\itemsep{0ex}
\item $\frac{p(\bb_1)}{p(\bb_2)} = e^{-\frac{\epsilon}{\epsilon'} (\|\bb_1\|_2 - \|\bb_2\|_2)} \leq e^{\frac{\epsilon}{\epsilon'} (\|\bb_1 - \bb_2\|_2)} \leq e^\epsilon$. The reverse direction can be obtained similarly.
\item The proof of Theorem 3.22 in \cite{dwork2011differential} applies using $\Delta_2(f) = \epsilon'$, giving that with probability at least $1-\delta$, we have that $e^{-\epsilon} \leq \frac{p(\bb_1)}{p(\bb_2)} \leq e^{\epsilon}$. Applying Theorem \ref{thm:epsilon_cd} gives the desired $(\epsilon,\delta)$-CR guarantee.
\end{enumerate}
\end{proof}

\begin{theorem}
Under the same regularity conditions of Theorem \ref{thm:asymptotic_bound}, we have that:
\begin{align*}
\| \nabla L(\bw^{(-m)}; \calD \setminus \calD_m) \|_2 &\leq \gamma (n-m) \left\| \left[ H_{\bw^*}^{(m)} \right]^{-1} \Delta^{(m)} \right\|_2^2 \leq \frac{4 \gamma m^2 C^2}{\lambda^2 (n-m)}.
\end{align*}
\end{theorem}

\begin{proof}
The proof is almost identical to that of Theorem \ref{thm:asymptotic_bound}, except that there are $n-m$ terms in the Hessian and $\Delta^{(m)}$ now scales linearly with $m$.
\end{proof}

\begin{theorem}
Suppose $\Phi$ is a randomized learning algorithm that is $(\epsilon_\text{DP}, \delta_\text{DP})$-differentially private, and the outputs of $\Phi$ are used in a linear model by minimizing $L_\bb$ and using a removal mechanism that guarantees $(\epsilon_\text{CR},\delta_\text{CR})$-certified removal. Then the entire procedure guarantees
$(\epsilon_\text{DP} \!+\! \epsilon_\text{CR}, \delta_\text{DP} \!+\! \delta_\text{CR})$-certified removal.
\end{theorem}

\begin{proof}
Let $\Phi$ be the randomized algorithm that learns a feature extractor from the
data $\calD$ and let $\mu(S) = P(\Phi(\calD) \in S)$ be the induced probability
measure over the space, $\Omega$, of all possible feature extractors.
Let $\calD' = \calD \setminus \bx$ be the dataset with $\bx$ removed and let $\mu'(\cdot)$ be
the corresponding probability measure for $\Phi(\calD')$.

Since $\Phi$ is $(\epsilon_\text{DP}, \delta_\text{DP})$-DP, for any $S \subseteq \Omega$, we have that:
$$\mu(S) = P(\Phi(\calD) \in S) \leq e^{\epsilon_\text{DP}} P(\Phi(\calD') \in S) = e^{\epsilon_\text{DP}} \mu'(S),$$
with probability $1 - \delta_\text{DP}$. In particular, this shows that $\mu$ is absolutely
continuous w.r.t. $\mu'$ and therefore admits a Radon-Nikodym derivative $g$.
Furthermore, $g$ is (almost everywhere w.r.t. $\mu'$) bounded by $e^{\epsilon_\text{DP}}$. Indeed,
suppose that there exists a set $S \subseteq \Omega$ with $\mu'(S) > 0$ such that
$g \geq e^{\epsilon_\text{DP}} + \alpha$ on $S$ for some $\alpha \geq 0$, then:
$$\mu(S) = \int_S g(S) \hspace{1ex} d\mu' \geq (e^{\epsilon_\text{DP}} + \alpha) \mu'(S) \geq
\mu(S) + \alpha \mu'(S),$$
which is a contradiction unless $\alpha = 0$.

Finally, for any $\phi \in \Omega$ such that $\Phi(\calD) = \phi$, let $A(\calD, \phi)$
be the learning algorithm that trains a model on $\calD$ using the feature extractor
$\phi$. Suppose that $M$ is an $(\epsilon_\text{CR}, \delta_\text{CR})$-CR mechanism, then by Fubini's Theorem:
\begin{align*}
P(M(A(\calD, \Phi(\calD)), \calD, \bx) \in \mathcal{T}) &= \int_\Omega P(M(A(\calD, \phi), \calD, \bx) \in \mathcal{T}) \hspace{1ex} d\mu \\
&\leq \int_\Omega e^{\epsilon_\text{CR}} P(A(\calD', \phi) \in \mathcal{T}) \hspace{1ex} d\mu \\
&= \int_\Omega e^{\epsilon_\text{CR}} P(A(\calD', \phi) \in \mathcal{T}) \cdot g \hspace{1ex} d\mu' \\
&\leq \int_\Omega e^{\epsilon_\text{DP} + \epsilon_\text{CR}} P(A(\calD', \phi) \in \mathcal{T}) \hspace{1ex} d\mu' \\
&= e^{\epsilon_\text{DP} + \epsilon_\text{CR}} P(A(\calD', \Phi(\calD')) \in \mathcal{T}),
\end{align*}
with probability at least $1 - \delta_\text{DP} - \delta_\text{CR}$. The lower bound can be shown in a similar fashion.
\end{proof}

\newpage
\section*{Errata}

The proof of Theorem \ref{thm:asymptotic_bound} (and similarly, Theorem \ref{thm:asymptotic_bound_batch}) requires upper bounding the norm of $\Delta = \lambda \bw^* + \nabla \ell((\bw^*)^\top \bx_n, y_n)$ for the minimizer $\bw^*$ of $L(\cdot; \calD)$. However, to apply this result to Theorem \ref{thm:removal_main}, it in fact requires upper bounding the norm of $\Delta$ for the minimizer $\bw^*$ of $L_\bb(\cdot; \calD)$, where $\bb$ is drawn from either the Laplace or the Gaussian distribution. This means the condition for $\bw^*$ in the proof of Theorem \ref{thm:asymptotic_bound} is:
\begin{align*}
0 &= \nabla L_\bb(\bw^*; \calD) = \sum_{i=1}^n \nabla \ell\left((\bw^*)^\top \bx_i, y_i\right) + \lambda n \bw^* + \bb \\
\Rightarrow \| \bw^* \|_2 &= \frac{\| \sum_{i=1}^n \nabla \ell((\bw^*)^\top \bx_i, y_i) - \bb \|_2}{\lambda n} \leq \frac{Cn + \|\bb\|_2}{\lambda n}.
\end{align*}
Since $\bb$ is a random vector, this quantity cannot be upper bounded in absolute terms, but one can obtain high probability bounds via concentration. Note that the data-dependent bounds in Corollary \ref{cor:data_dependent} (and similarly, Corollary \ref{cor:data_dependent_batch}) are unaffected because $\|\Delta\|_2$ is computed directly rather than upper bounded. Thanks to Lu Yi for finding this error!

\end{document}